\documentclass{article}
\usepackage{nips12submit_e,times}

\title{Mirror Descent Meets Fixed Share \\ (and feels no regret)}

\author{
Nicol\`{o} Cesa-Bianchi \\
Universit\`{a} degli Studi di Milano \\
\texttt{nicolo.cesa-bianchi@unimi.it}
\And
Pierre Gaillard \\
Ecole Normale Sup\'{e}rieure\thanks{Ecole Normale Sup\'{e}rieure, Paris -- CNRS -- INRIA, within the project-team CLASSIC}, Paris \\
\texttt{pierre.gaillard@ens.fr}
\And
G\'{a}bor Lugosi \\
ICREA \& Universitat Pompeu Fabra, Barcelona \\
\texttt{gabor.lugosi@upf.edu}
\And
Gilles Stoltz \\
Ecole Normale Sup\'{e}rieure\footnotemark[1], Paris
{\&}\\ HEC Paris, Jouy-en-Josas, France \\
\texttt{gilles.stoltz@ens.fr}
}

\usepackage{amsmath,amssymb,amsthm}
\usepackage{algorithm2e}
\usepackage[utf8]{inputenc}
\usepackage[T1]{fontenc}
\usepackage{times}
\usepackage{xcolor}
\usepackage{cancel}
\usepackage{dsfont}
\usepackage{mycommands}

\newtheorem{theorem}{Theorem}
\newtheorem{lemma}{Lemma}

\newtheorem{corollary}{Corollary}

\newcommand{\argmin}{\mathop{\mathrm{argmin}}}
\renewcommand{\leq}{\leqslant}
\renewcommand{\le}{\leqslant}

\renewcommand{\geq}{\geqslant}
\newcommand{\loss}{\ell}
\newcommand{\bloss}{\mathbf{\ell}}
\renewcommand{\bp}{\widehat{\mathbf{p}}}

\allowdisplaybreaks
\nipsfinalcopy

\begin{document}

\maketitle

\begin{abstract}
Mirror descent with an entropic regularizer is known to achieve shifting regret bounds that are logarithmic
in the dimension. This is done using either a carefully designed projection or by a weight sharing
technique. Via a novel unified analysis, we show that these two approaches deliver essentially equivalent bounds
on a notion of regret generalizing shifting, adaptive, discounted, and other related regrets. Our analysis also captures
and extends the generalized weight sharing technique of Bousquet and Warmuth, and can be refined in
several ways, including improvements for small losses and adaptive tuning of parameters.
\end{abstract}

\section{Introduction}

Online convex optimization is a sequential prediction paradigm in which,
at each time step, the learner chooses an element from a fixed convex
set $\scS$ and then is given access to a convex loss function defined on the
same set. The value of the function on the chosen element is
the learner's loss. Many problems such as prediction with expert
advice, sequential investment, and online regression/classification
can be viewed as special cases of this general framework.
Online learning algorithms are designed to minimize the regret. The
standard notion of regret is the difference between the learner's
cumulative loss and the cumulative loss of the single best element in
$\scS$. A much harder criterion to minimize is shifting regret, which
is defined as the difference between the learner's cumulative loss and
the cumulative loss of an arbitrary sequence of elements in
$\scS$. Shifting regret bounds are typically expressed in terms of the
\textsl{shift}, a notion of regularity measuring the length of the trajectory in $\scS$ described by
the comparison sequence (i.e., the sequence of elements against which the regret is evaluated).
In online convex optimization, shifting regret bounds for convex subsets $\scS\subseteq\R^d$ are obtained for the projected online mirror descent (or follow-the-regularized-leader) algorithm. In this case the shift is typically computed in terms of the $p$-norm of the difference of consecutive elements in the comparison sequence ---see~\cite{hewa01,Zink03} and~\cite{CeLu}.

We focus on the important special case when $\scS$ is the simplex. In~\cite{hewa01} shifting bounds are shown for projected mirror descent with entropic regularizers using a $1$-norm to measure the shift.\footnote{Similar $1$-norm shifting bounds can also be proven using the analysis of~\cite{Zink03}. However, without using entropic regularizers it is not clear how to achieve a logarithmic dependence on the dimension, which is one of the advantages of working in the simplex.} When the comparison sequence is restricted to the corners of the simplex (which is the setting of prediction with expert advice), then the shift is naturally defined to be the number of times the trajectory moves to a different corner. This problem is often called ``tracking the best expert'' ---see, e.g., \cite{hewa98,Vov99,hewa01,bowa,GyLiLu05a}, and it is well known that exponential weights with weight sharing, which corresponds to the fixed-share algorithm of~\cite{hewa98}, achieves a good shifting bound in this setting. In~\cite{bowa} the authors introduce a generalization of the fixed-share algorithm, and prove various shifting bounds for any trajectory in the simplex. However, their bounds are expressed using a quantity that corresponds to a proper shift only for trajectories on the simplex corners.

In this paper we offer a unified analysis of mirror descent, fixed share, and the generalized fixed share of~\cite{bowa} for the setting of online convex optimization in the simplex. Our bounds are expressed in terms of a notion of shift based on the total variation distance. Our analysis relies on a generalized notion of shifting regret which includes, as special cases, related notions of regret such as adaptive regret, discounted regret, and regret with time-selection functions. Perhaps surprisingly, we show that projected mirror descent and fixed share achieve essentially the same generalized regret bound. Finally, we show that widespread techniques in online learning, such as improvements for small losses and adaptive tuning of parameters, are all easily captured by our analysis.

\section{Preliminaries}
\label{sec:prelim}
For simplicity, we derive our results in the setting of online linear optimization. As we show in the supplementary material, these results can be easily extended to the more general setting of online convex optimization through a standard linearization step.

Online linear optimization may be cast as a repeated game between the {\sl forecaster} and the {\sl environment} as follows. We use $\Delta_d$ to denote the simplex $\bigl\{ \bq \in [0,1]^d \,:\, \|\bq\|_1 = 1 \bigr\}$.

\noindent
{\bf Online linear optimization in the simplex.} For each round $t = 1, \dots, T$, \\
\indent 1. Forecaster chooses $\bp_t=(\wh{p}_{1,t},\ldots,\wh{p}_{d,t})\in \Delta_d$ \\
\indent 2. Environment chooses a loss vector $\bloss_t = (\loss_{1,t},\ldots,\loss_{d,t}) \in [0,1]^d$ \\
\indent 3. Forecaster suffers loss $\bp_t^{\top}\bloss_t$~.

The goal of the forecaster is to minimize the accumulated loss, e.g.,
$\wh{L}_T= \sum_{t=1}^T \bp_t^{\top}\bloss_t$.
In the now classical problem of prediction with expert advice,
the goal of the forecaster is to compete with the best fixed component
(often called ``expert'')
chosen in hindsight, that is, with $\min_{i=1,\ldots,T} \sum_{t=1}^T \loss_{i,t}$;
or even to compete with a richer class of {\sl sequences} of components.
In Section~\ref{sec:target} we state more specifically the goals considered in this paper.

We start by introducing our main algorithmic tool, described in
Figure \ref{alg:genupd}, a share algorithm whose formulation
generalizes the seemingly unrelated formulations of the algorithms studied in~\cite{hewa98,hewa01,bowa}. It is parameterized by the ``mixing functions''
$\psi_{t}:[0,1]^{t d} \to \Delta_d$ for $t \geq 2$
that assign probabilities to past ``pre-weights'' as defined below.
In all examples discussed in this paper, these mixing functions are
quite simple, but working with such a general model makes the main ideas more
transparent. We then provide a simple lemma that serves as the starting
point\footnote{We only deal with linear losses in this paper. However, it is straightforward that
for sequences of $\eta$--exp-concave loss functions, the additional
term $\eta/8$ in the bound is no longer needed.}
for analyzing different instances of this generalized share algorithm.

\begin{algorithm}[h]
\begin{minipage}{\textwidth}
\rule{\linewidth}{.5pt}
\caption{\label{alg:1} The generalized share algorithm.}
\label{alg:genupd}
\medskip
	\textbf{Parameters:} learning rate $\eta>0$ and mixing functions $\psi_{t}$ for $t \geq 2$ \vspace{.125cm} \\
	\textbf{Initialization:} $\bp_1 = \mathbf{v}_1 = (1/d,\dots,1/d)$ \vspace{.125cm} \\
	\textbf{For} each round $t = 1, \dots, T$,
	\begin{itemize}
	 \nitem{1} Predict $\bp_t$ ;
	 \nitem{2} Observe loss $\bloss_t \in [0,1]^d$ ;
	 \nitem{3} [loss update] \textbf{For} each $j=1,\dots,d$ define \\
	 $\displaystyle{v_{j,t+1} = \frac{\wh{p}_{j,t} \, e^{-\eta\,\loss_{j,t}}}{\sum_{i=1}^d \wh{p}_{i,t} \, e^{-\eta\,\loss_{i,t}}}} \ \ \ $ the current pre-weights, \ \ \ and $\bv_{t+1} = (v_{1,t+1},\ldots,v_{d,t+1})$; \\[2mm]
	 $ \uV = \bigl[v_{i,s}\bigr]_{1 \leq i \leq d, \, 1\leq s\leq t+1} \quad$ the $d \times (t+1)$
matrix of all past and current pre-weights;
	 \nitem{4} [shared update] Define $\displaystyle{\bp_{t+1} = \psi_{t+1} \bigl(\uV \bigr)}$.
	\end{itemize}
\rule{\linewidth}{.5pt} \vspace{.125cm}
\end{minipage}
\end{algorithm}

\begin{lemma}
\label{lem:b0}
For all $t\geq 1$ and for all $\bq_t\in\Delta_d$, Algorithm~\ref{alg:genupd} satisfies
\[
    \bigl(\bp_t - \bq_t\bigr)^{\! \top}\bloss_t
\le
    \frac{1}{\eta} \sum_{i=1}^d q_{i,t} \ln \frac{v_{i,t+1}}{\wh{p}_{i,t}} + \frac{\eta}{8}~.
\]
\end{lemma}
\begin{proof}
By Hoeffding's inequality (see, e.g., \cite[Section~A.1.1]{CeLu}),
\begin{equation}
\label{eq:convexhoeffding}
    \sum_{j=1}^d \wh{p}_{j,t}\,\loss_{j,t}
    \le
    - \frac{1}{\eta} \ln \left( {\sum_{j=1}^d \wh{p}_{j,t}\, e^{-\eta\,\loss_{j,t}}} \right) + \frac{\eta}{8}\,.
\end{equation}
By definition of $v_{i,t+1}$, for all $i=1,\dots,d$ we then have $\sum_{j=1}^d \wh{p}_{j,t}\,e^{-\eta\,\loss_{j,t}} = \wh{p}_{i,t}\,e^{-\eta\,\loss_{i,t}}/v_{i,t+1},$
which implies $\quad\bp_t^{\top}\bloss_t \le
\loss_{i,t} + (1/\eta) \ln (v_{i,t+1}/\wh{p}_{i,t}) + \eta/8$.
The proof is concluded by taking a convex aggregation with respect to $\bq_t$.
\end{proof}

\section{A generalized shifting regret for the simplex}
\label{sec:target}

We now introduce a generalized notion of shifting regret which unifies and generalizes the notions of discounted regret (see~\cite[Section~2.11]{CeLu}), adaptive regret (see~\cite{hase09}), and shifting regret (see~\cite{Zink03}). For a fixed horizon $T$, a sequence of discount factors $\beta_{t,T} \geq 0$ for $t=1,\dots,T$ assigns varying weights to the instantaneous losses suffered at each round.
% We consider increasing or decreasing sequences are to be considered.
We compare the total loss of the forecaster with the loss of an arbitrary sequence of vectors $\bq_1,\ldots,\bq_T$ in the simplex $\Delta_d$. Our goal is to bound the regret
\[
\sum_{t=1}^T \beta_{t,T} \, \bp_t^{\top} \bloss_t - \sum_{t=1}^T \beta_{t,T} \, \bq_t^{\top} \bloss_t
\]
in terms of the ``regularity'' of the comparison sequence $\bq_1,\ldots,\bq_T$ and of the variations of the discounting weights $\beta_{t,T}$. By setting $\bu_t = \beta_{t,T} \, \bq_t^{\top} \in \R_+^d$, we can rephrase the above regret as
%Our target can be rephrased in terms of a comparison against arbitrary sequences $\bu_1,\ldots,\bu_T \in \R_+^d$ of vectors with non-negative components, where the target is to bound quantities of the form
\begin{equation}
\label{eq:defregr}
\sum_{t=1}^T \|\bu_t\|_1 \, \bp_t^{\top} \bloss_t  - \sum_{t=1}^T \bu_t^{\top} \bloss_t~.
\end{equation}
%in terms of the ``regularity'' of the comparison sequence $\bu_1,\ldots,\bu_T$.
In the literature on tracking the best expert~\cite{hewa98,Vov99,hewa01,bowa}, the
regularity of the sequence $\bu_1,\ldots,\bu_T$ is measured as the number of
times $\bu_t \neq \bu_{t+1}$.
%(henceforth referred to as ``hard shifts'').
We introduce the following regularity measure
\begin{equation}
\label{eq:regularity}
m(\b u_1^T) =  \sum_{t=2}^{T} D_{\mathrm{TV}}(\bu_{t},\bu_{t-1})
\end{equation}
where for $\bx=(x_1,\ldots,x_d),\by=(y_1,\ldots,y_d)\in \R_+^d$, we define
$D_{\mathrm{TV}}(\bx,\by) =  \sum_{x_i \geq y_i} (x_i - y_i)$.
Note that when $\bx,\by\in \Delta_d$, we recover the total variation
distance $D_{\mathrm{TV}}(\bx,\by) = \tfrac{1}{2} \norm[\bx-\by]_1$, while for general $\bx,\by\in \R_+^d$, the quantity $D_{\mathrm{TV}}(\bx,\by)$ is not necessarily
symmetric and is always bounded by $\norm[\bx-\by]_1$.
%Note also that when the vectors $\bu_t$ are incidence vectors $(0,\dots,0,1,0,\dots,0) \in \R^d$ of elements $i_t \in \{ 1,\ldots,d \}$, then $m(\bu_1^T)$ corresponds to the number of shifts of the sequence $i_1^T\in \{1,\ldots,d\}^T$.
The traditional shifting regret of~\cite{hewa98,Vov99,hewa01,bowa} is obtained from \eqref{eq:defregr} when all $\bu_t$ are such that $\norm[\bu_t]_1 = 1$.

\section{Projected update}
The shifting variant of the EG algorithm analyzed in~\cite{hewa01} is a special case of the generalized share algorithm in which the function $\psi_{t+1}$ performs a projection of the pre-weights on the convex set $\Delta_d^\alpha = [{\alpha}/{d},1]^d \cap \Delta_d$. Here $\alpha \in (0,1)$ is a fixed parameter. We can prove (using techniques
similar to the ones shown in the next section---see the
supplementary material) the following bound which generalizes~\cite[Theorem~16]{hewa01}.
\begin{theorem}
\label{th:HW}
For all $T \geq 1$, for all sequences $\bloss_1,\dots,\bloss_t \in [0,1]^d$ of loss vectors, and for all $\bu_1,\dots,\bu_T \in \R_+^d$, if Algorithm~1 is run with the above update, then
\begin{equation}
\label{eq:zinlin}
\sum_{t=1}^T \|\bu_t\|_1 \, \bp_t^{\top} \bloss_t  - \sum_{t=1}^T \bu_t^{\top} \bloss_t
\leq \frac{\norm[\bu_1]_1 \ln d}{\eta} + \frac{m(\bu_1^T)}{\eta}\ln \frac{d }{\alpha}  + \left(\frac{\eta}{8} + {\alpha }\right) \sum_{t=1}^T\norm[\bu_t]_1~.
\end{equation}
%where $m(\bu_1^T) = \sum_{t=2}^{T} D_{\mathrm{TV}}(\bu_{t}, \bu_{t-1})$.
\end{theorem}
This bound can be optimized by a proper tuning of $\alpha$ and $\eta$ parameters. We show a similarly tuned (and slightly better) bound in Corollary~\ref{cor:dtvopt}.

\section{Fixed-share update}
\label{sec:FS}
Next, we consider a different instance of the generalized share algorithm corresponding to the update
\begin{equation}
\label{eq:shareupdateclassicalfs}
    \wh{p}_{j,t+1} = \sum_{i=1}^d \left(\frac{\alpha}{d} + (1-\alpha) \ind_{i=j}\right) v_{i,t+1}
= \frac{\alpha}{d} + (1-\alpha) v_{j,t+1}\,,
\qquad
    0 \le \alpha \le 1
\end{equation}
Despite seemingly different statements, this update in Algorithm~\ref{alg:genupd} can be seen to lead {\sl exactly} to the fixed-share algorithm of~\cite{hewa98} for prediction with expert advice. We now show that this update delivers a bound on the regret almost equivalent to (though slightly better than)
that achieved by projection on the subset $\Delta_d^{\alpha}$ of the simplex.
\begin{theorem}
\label{prop:dtv}
With the above update, for all $T \geq 1$, for all sequences $\bloss_1,\dots,\bloss_T$ of loss
vectors $\bloss_t\in [0,1]^d$, and for all $\bu_1,\dots,\bu_T \in \R_+^d$,
\begin{multline*}
\sum_{t=1}^T  \norm[\bu_t]_1\,\bp_t^{\top}\bloss_t - \sum_{t=1}^T \bu_t^{\top}\bloss_t
\le
    \frac{\|\bu_1\|_1 \ln d}{\eta} + \frac{\eta}{8}
\sum_{t=1}^T \norm[\b u_t]_1
\\
    + \frac{m(\bu_1^T)}{\eta} \ln \frac{d}{\alpha} + \frac{\sum_{t=2}^T \norm[\b u_t]_1-m(\bu_1^T)}{\eta} \ln \frac{1}{1-\alpha}~.
\end{multline*}
\end{theorem}
Note that if we only consider vectors of the form $\bu_t = \bq_t = (0,\ldots,0,1,0,\ldots,0)$ then
$m(\bq_1^T)$ corresponds to the number of times $\bq_{t+1} \neq \bq_t$ in the sequence $\bq_1^T$. We thus recover
\cite[Theorem~1]{hewa98} and \cite[Lemma 6]{bowa} from the much more general Theorem~\ref{prop:dtv}.

The fixed-share forecaster does not need to ``know''
anything in advance about the sequence of the norms $\|\bu_t\|$
for the bound above to be valid. Of course, in order to
minimize the obtained upper bound, the tuning parameters $\alpha, \, \eta$ need to
be optimized and their values will depend on the maximal values of
$m(\bu_1^T)$ and $\sum_{t=1}^T \norm[\b u_t]_1$ for the sequences one wishes to compete against.
This is illustrated in the following corollary, whose proof is omitted.
Therein, $h(x)=-x\ln x -(1-x)\ln(1-x)$ denotes the binary entropy function for
$x\in [0,1]$. We recall\footnote{As can be seen
by noting that $\ln\bigl(1/(1-x)\bigr) < x/(1-x)$}
that $h(x) \leq x \ln(e/x)$ for $x \in [0,1]$.
\begin{corollary}
\label{cor:dtvopt}
Suppose Algorithm~\ref{alg:genupd} is run with the
update~(\ref{eq:shareupdateclassicalfs}).
Let $m_0 > 0$ and $U_0 > 0$. For all $T \geq 1$, for all sequences $\bloss_1,\dots,\bloss_T$
of loss vectors $\bloss_t\in [0,1]^d$, and for all sequences
$\bu_1,\dots,\bu_T \in \R_+^d$ with $\norm[\bu_1]_1 + m(\bu_1^T) \leq m_0$
and $\sum_{t=1}^T \norm[\b u_t]_1 \leq U_0$,
\[
\sum_{t=1}^T  \norm[\bu_t]_1\,\bp_t^{\top}\bloss_t - \sum_{t=1}^T \bu_t^{\top}\bloss_t
\le
   \sqrt{\frac{U_0}{2}
\Biggl( m_0 \ln d + U_0 \, h\!\left(\frac{m_0}{U_0}\right)\Biggr)}
\le \sqrt{\frac{U_0 \, m_0}{2}
\Biggl( \ln d + \ln \left(\frac{e \, U_0}{m_0}\right)\Biggr)}
\]
whenever $\eta$ and $\alpha$ are optimally chosen in terms of $m_0$ and $U_0$.
\end{corollary}
\begin{proof}[Proof of Theorem~\ref{prop:dtv}]
Applying Lemma~\ref{lem:b0} with $\b q_t = \b u_t / \norm[\b u_t]_1$, and multiplying by $\norm[\b u_t]_1$, we get for all $t \geq 1$ and $\bu_t \in \R_+^d$
\begin{equation}
\label{eq:proofdtv0}
    \norm[\b u_t]_1 \,\bp_t^{\top}\bloss_t - \bu_t^{\top}\bloss_t
\le
    \frac{1}{\eta} \sum_{i=1}^d u_{i,t} \ln \frac{v_{i,t+1}}{\wh{p}_{i,t}} + \frac{\eta}{8} \norm[\b u_t]_1~.
\end{equation}
We now examine
\begin{equation}
\label{eq:proofdtv1}
\sum_{i=1}^d u_{i,t} \ln \frac{v_{i,t+1}}{\wh{p}_{i,t}}  =  \sum_{i=1}^d \left(u_{i,t} \ln \frac{1}{\wh{p}_{i,t}} - u_{i,t-1} \ln \frac{1}{v_{i,t}} \right) + \sum_{i=1}^d \left( u_{i,t-1} \ln \frac{1}{v_{i,t}} - u_{i,t} \ln \frac{1}{v_{i,t+1}} \right)~.
\end{equation}
For the first term on the right-hand side, we have
\begin{align}
\nonumber
    \sum_{i=1}^d \left( u_{i,t} \ln \frac{1}{\wh{p}_{i,t}} - u_{i,t-1} \ln \frac{1}{v_{i,t}} \right)
& =
    \sum_{i\,:\,u_{i,t} \geq u_{i,t-1}} \left(
    \left(u_{i,t}-u_{i,t-1}\right) \ln \frac{1}{\wh{p}_{i,t}} + u_{i,t-1} \ln \frac{v_{i,t}}{\wh{p}_{i,t}} \right)
\\ & \quad+
\label{eq:proofdtv2}
    \sum_{i\,:\,u_{i,t} < u_{i,t-1}} \biggl( \underbrace{\left(u_{i,t}-u_{i,t-1}\right) \ln \frac{1}{v_{i,t}}}_{\leq 0} + u_{i,t} \ln \frac{v_{i,t}}{\wh{p}_{i,t}} \biggr).
\end{align}
In view of the update~\eqref{eq:shareupdateclassicalfs}, we have $1/\wh{p}_{i,t} \leq d/\alpha$
and $v_{i,t}/\wh{p}_{i,t} \leq 1/(1-\alpha)$. Substituting in~\eqref{eq:proofdtv2}, we get
\begin{align*}
    \sum_{i=1}^d &\left( u_{i,t} \ln \frac{1}{\wh{p}_{i,t}} - u_{i,t-1} \ln \frac{1}{v_{i,t}} \right)
\\ &\le
    \sum_{i\,:\,u_{i,t} \geq u_{i,t-1}} \left(u_{i,t}-u_{i,t-1}\right) \ln \frac{d}{\alpha} + \left( \sum_{i:~u_{i,t} \geq u_{i,t-1}} u_{i,t-1}  + \sum_{i:~u_{i,t} < u_{i,t-1}}  u_{i,t} \right) \ln \frac{1}{1-\alpha}
\\ & =
    D_{\mathrm{TV}}(\bu_{t},\bu_{t-1}) \ln \frac{d}{\alpha} + \underbrace{\left( \sum_{i=1}^d u_{i,t} - \sum_{i\,:\,u_{i,t} \geq u_{i,t-1}}  (u_{i,t}-
    u_{i,t-1}) \right)}_{= \norm[\b u_t]_1 - D_{\mathrm{TV}}(\bu_{t},\bu_{t-1})} \ln \frac{1}{1-\alpha}\,.
\end{align*}
The sum of the second term in~\eqref{eq:proofdtv1} telescopes. Substituting the
obtained bounds in the first sum of the right-hand side in~\eqref{eq:proofdtv1}, and summing over $t=2,\dots,T$, leads to
\begin{multline}
\nonumber
\sum_{t=2}^T \sum_{i=1}^d u_{i,t} \ln \frac{v_{i,t+1}}{\wh{p}_{i,t}}  \leq  m(\bu_1^T) \ln \frac{d}{\alpha} + \left(
\sum_{t=2}^T \norm[\b u_t]_1-m(\bu_1^T)\right) \ln \frac{1}{1-\alpha}
\\ + \sum_{i=1}^d u_{i,1} \ln \frac{1}{v_{i,2}} - \underbrace{u_{i,T} \ln \frac{1}{v_{i,T+1}}}_{\leq 0}\,.
\end{multline} \vspace{-1.25cm} \ \\

\noindent
We hence get from~\eqref{eq:proofdtv0}, which we use in particular for $t=1$,
\begin{multline*}
    \sum_{t=1}^T \norm[\b u_t]_1 \bp_t^{\top}\bloss_t - \bu_t^{\top} \bloss_t
\le
	\frac{1}{\eta} \sum_{i=1}^d u_{i,1} \ln \frac{1}{\wh{p}_{i,1}} + \frac{\eta}{8} \sum_{t=1}^T \norm[\bu_t]_1
\\+
	 \frac{m(\bu_1^T)}{\eta} \ln \frac{d}{\alpha} + \frac{\sum_{t=2}^T \norm[\b u_t]_1
m(\bu_1^T)}{\eta} \ln \frac{1}{1-\alpha}\,. \quad
\end{multline*}  \vspace{-1.2cm}

\end{proof}

\section{Applications}

%%%%%%%%%%%%%%%%%%%%%%%%%%%%%%%%%%%%%%%%%%%%%%%%%%%%%%%%

%\subsection{Shifting regret / Improvements in high dimensions}

%In this setting the bound proved in Corollary~\ref{cor:dtvopt} for the fixed-share forecaster improves over the bound stated in Theorem~\ref{th:Zink} for the greedy projection forecaster as the former replaces the $\sqrt{d}$ of the latter factor by something of the order of $\sqrt{\ln d + \ln T}$. This is particularly important in situations where there are many experts. In a sense, the fixed-share forecaster lifts to the shifting scenario the $\ln d$ advantage of Winnow/EG over standard gradient descent.

%For the case where $d$ is even much larger than the number of rounds, $T \ll d$, the fixed-share forecaster should be replaced by another forecaster, introduced by~\cite{bowa} and for which we propose an analysis in terms of shifting regret in Section~\ref{sec:bowa} below.

%\subsection{Adaptive regret}

We now show how our regret bounds can be specialized to obtain bounds on adaptive and discounted regret, and on regret with time-selection functions. We show regret bounds only for the specific instance of the generalized share algorithm using update~(\ref{eq:shareupdateclassicalfs});
but the discussion below also holds up to minor modifications for the forecaster studied in Theorem~\ref{th:HW}.

\paragraph{Adaptive regret} was introduced by \cite{hase09} and can be viewed as a variant of discounted regret where the monotonicity assumption is dropped. For $\tau_0 \in \{1,\ldots,T\}$, the {\sl $\tau_0$-adaptive regret} of a forecaster is defined by
\begin{equation}
\label{def:adapt}
\c R_T^{\tau_0-\mathrm{adapt}}
= \max_{\scriptsize \begin{array}{c} [r,s] \subset [1,T] \\ s+1-r \leq \tau_0 \end{array}}
\left\{ \sum_{t=r}^{s} \bp_t^{\top}\bloss_t
 - \min_{\bq \in \Delta_d} \sum_{t=r}^{s} \bq^{\top}\bloss_t \right\}\,.
\end{equation}
The fact that this is a special case of~\eqref{eq:defregr} clearly emerges from the proof of Corollary~\ref{cor:adapt} below here.

Adaptive regret is an alternative way to measure the performance of a
forecaster against a changing environment.
It is a straightforward observation that adaptive regret bounds also lead to shifting regret bounds (in terms of
hard shifts). In this paper we note that these two notions of regret share an even tighter connection,
as they can be both viewed as instances of the same \textsl{alma mater} notion of regret, i.e.,
the generalized shifting regret introduced in Section~\ref{sec:target}.
The work \cite{hase09} essentially considered the case of online convex optimization
with exp-concave loss function; in case of general convex functions, they also mentioned that the greedy projection forecaster of \cite{Zink03} enjoys adaptive regret guarantees. This is obtained in much the same way as we obtain an adaptive regret bound for the fixed-share forecaster in the next result.
%This can be seen by instantiating the bound of Theorem~\ref{th:Zink} as we do below to get an adaptive regret bound for the fixed-share forecaster based on Corollary~\ref{cor:dtvopt}.
%
\begin{corollary}
\label{cor:adapt}
Suppose that Algorithm~\ref{alg:genupd} is run with the shared update~(\ref{eq:shareupdateclassicalfs}). Then
for all $T \geq 1$, for all sequences $\bloss_1,\dots,\bloss_T$ of loss
vectors $\bloss_t\in [0,1]^d$, and for all $\tau_0 \in \{1,\ldots,T\}$,
\[
\c R_T^{\tau_0-\mathrm{adapt}}
\leq \sqrt{\frac{\tau_0}{2} \left( \tau_0\,h\! \left(\frac{1}{\tau_0}\right) + \ln d \right)}
\leq \sqrt{\frac{\tau_0}{2} \, \ln (e d \tau_0)}
\]
whenever $\eta$ and $\alpha$ are chosen optimally (depending
on $\tau_0$ and $T$).
\end{corollary}

As mentioned in~\cite{hase09}, standard lower bounds on the regret show that the obtained
bound is optimal up to the logarithmic factors.

\begin{proof}
For $1 \leq r \leq s \leq T$ and $\bq \in \Delta_d$, the regret in the
right-hand side of~\eqref{def:adapt} equals the regret considered
in Theorem~\ref{prop:dtv} against the sequence $\bu_1^T$
defined as $\bu_t = \bq$ for $t = r,\dots,s$ and $\b 0 = (0,\dots,0)$ for the remaining $t$.
When $r \geq 2$, this sequence is such that
$D_{\mathrm{TV}}(\bu_r,\bu_{r-1}) = D_{\mathrm{TV}}(\bq,\b0) = 1$ and $D_{\mathrm{TV}}(\bu_{s+1},\bu_{s})
= D_{\mathrm{TV}}(\b0,\bq) = 0$ so that $m(\bu_1^T) = 1$, while $\norm[\bu_1]_1 = 0$.
When $r = 1$, we have $\norm[\bu_1]_1 = 1$ and $m(\bu_1^T) = 0$. In all cases,
$m(\bu_1^T) + \norm[\bu_1]_1 = 1$, that is, $m_0 = 1$.
Specializing the bound of Theorem~\ref{prop:dtv} with the additional choice $U_0 = \tau_0$
gives the result.
\end{proof}

%\subsection{Discounted regret and time-selection functions}

\paragraph{Discounted regret} was introduced in~\cite[Section 2.11]{CeLu} and is defined by
\begin{equation}
\label{eq:defdisc}
\max_{\bq \in \Delta_d} \, \sum_{t=1}^T \beta_{t,T} \, \bigl( \bp_t^{\top} \bloss_t - \bq^{\top} \bloss_t \bigr)~.
\end{equation}
The discount factors $\beta_{t,T}$ measure the relative importance of more recent losses to older losses.
For instance, for a given horizon $T$, the discounts $\beta_{t,T}$ may be larger as $t$ is closer to $T$.
On the contrary, in a game-theoretic setting, the earlier losses may matter more then the more recent ones (because of interest
rates), in which case $\beta_{t,T}$ would be smaller as $t$ gets closer to $T$. We mostly consider below
monotonic sequences of discounts (both non-decreasing and non-increasing). Up to a normalization,
we assume that all discounts $\beta_{t,T}$ are in $[0,1]$.
As shown in~\cite{CeLu}, a minimal requirement to get non-trivial bounds is that the
sum of the discounts satisfies $U_T = \sum_{t \leq T} \beta_{t,T} \to \infty$ as $T \to \infty$.

A natural objective is to show that the quantity in~\eqref{eq:defdisc} is $o(U_T)$,
for instance, by bounding it by something of the order of $\sqrt{U_T}$.
We claim that Corollary~\ref{cor:dtvopt} does so, at least
whenever the sequences $(\beta_{t,T})$ are monotonic for all $T$.
To support this claim, we only need to show that $m_0 = 1$ is a suitable value to
deal with~\eqref{eq:defdisc}. Indeed, for all $T \geq 1$ and for all $\bq \in \Delta_d$,
the measure of regularity involved in the corollary satisfies \vspace{-.7cm}

\[
\norm[\beta_{1,T}\bq]_1 + m\bigl( (\beta_{t,T}\bq)_{t \leq T} \bigr)
= \beta_{1,T} + \sum_{t=2}^T \bigl( \beta_{t,T}- \beta_{t-1,T} \bigr)_+
= \max\bigl\{\beta_{1,T}, \, \beta_{T,T} \bigr\} \leq 1\,, \vspace{-.3cm}
\]

\noindent where the second equality follows from the monotonicity assumption on the discounts.

The values of the discounts for all $t$ and $T$ are usually known in advance. However, the horizon $T$
is not. Hence, a calibration issue may arise. The online tuning of the parameters $\alpha$ and $\eta$ shown in
Section~\ref{sec:online} entails a forecaster that can get discounted regret bounds of the order $\sqrt{U_T}$
for all $T$. The fundamental reason for this is that the discounts only come in the definition of the fixed-share
forecaster via their sums. In contrast, the forecaster discussed in \cite[Section 2.11]{CeLu} weighs each instance $t$
directly with $\beta_{t,T}$ (i.e., in the very definition of the forecaster)
and enjoys therefore no regret guarantees for horizons other than $T$ (neither before $T$
nor after $T$). Therein, the knowledge of the horizon $T$ is so crucial that it cannot be dealt with easily, not even with
online calibration of the parameters or with a doubling trick. We insist that for the fixed-share forecaster,
much flexibility is gained as some of the discounts $\beta_{t,T}$ can change in a drastic manner for a round $T$
to values $\beta_{t,T+1}$ for the next round. However we must admit that the bound of \cite[Section 2.11]{CeLu}
is smaller than the one obtained above, as it of the order of $\sqrt{\sum_{t \leq T} \beta_{t,T}^2}$,
in contrast to our $\sqrt{\sum_{t \leq T} \beta_{t,T}}$ bound. Again, this improvement was made possible because of the knowledge
of the time horizon.

As for the comparison to the setting of discounted losses of~\cite{disc},
we note that the latter can be cast as a special case of our setting (since the discounting weights take
the special form  $\beta_{t,T} = \gamma_t \ldots \gamma_{T-1}$ therein, for some sequence $\gamma_s$ of positive numbers).
In particular, the fixed-share forecaster can satisfy the bound stated in~\cite[Theorem 2]{disc}, for instance,
by using the online tuning techniques of Section~\ref{sec:online}.
A final reference to mention is the setting of time-selection functions of \cite[Section~6]{blma07},
which basically corresponds to knowing in advance the weights $\norm[\bu_t]_1$ of the comparison
sequence $\bu_1,\ldots,\bu_T$ the forecaster will be evaluated against. We thus generalize their results as well.

\section{Refinements and extensions}

We now show that techniques for refining the standard online analysis can be easily applied to our framework.
We focus on the following: improvement for small losses, sparse target sequences, and dynamic tuning of parameters. Not all of them where within reach of previous analyses.

\subsection{Improvement for small losses}

The regret bounds of the fixed-share forecaster can
be significantly improved when the cumulative loss of the best sequence of experts is small. The next result improves on Corollary~\ref{cor:dtvopt} whenever $L_0 \ll U_0$. For concreteness, we focus on the fixed-share update~(\ref{eq:shareupdateclassicalfs}).
\begin{corollary}
\label{cor:smalllosses}
Suppose Algorithm~\ref{alg:genupd} is run with the
update~(\ref{eq:shareupdateclassicalfs}).
Let $m_0 > 0$, $U_0 > 0$, and $L_0 >0$. For all $T \geq 1$, for all sequences $\bloss_1,\dots,\bloss_T$
of loss vectors $\bloss_t\in [0,1]^d$, and for all sequences
$\bu_1,\dots,\bu_T \in \R_+^d$ with $\norm[\bu_1]_1 + m(\bu_1^T) \leq m_0$,
$\sum_{t=1}^T \norm[\b u_t]_1 \leq U_0$, and $\sum_{t=1}^T \bu_t^\top \bloss_t \leq L_0$,
\[
\sum_{t=1}^T  \norm[\bu_t]_1\,\bp_t^{\top}\bloss_t - \sum_{t=1}^T \bu_t^{\top}\bloss_t
\le \sqrt{L_0 \, m_0
\Biggl( \ln d + \ln \left(\frac{e \, U_0}{m_0}\right)\Biggr)} + \ln d + \ln \left(\frac{e \, U_0}{m_0}\right)
\]
whenever $\eta$ and $\alpha$ are optimally chosen in terms of $m_0$, $U_0$, and $L_0$.
\end{corollary}
Here again, the parameters $\alpha$ and $\eta$ may be tuned online using the techniques shown in Section~$\ref{sec:online}$.
The above refinement is obtained by mimicking the analysis of Hedge forecasters for  small losses (see, e.g., \cite[Section~2.4]{CeLu}).
In particular, one should substitute Lemma~\ref{lem:b0} with the following lemma in the analysis carried out in
Section~\ref{sec:FS}; its proof follows from the mere replacement of Hoeffding's inequality
by \cite[Lemma A.3]{CeLu}, which states that for all $\eta \in \R$ and for all random variable $X$ taking values in $[0,1]$,
one has $\ln \E[e^{-\eta X}] \leq (e^{-\eta} -1) \E X$.
\begin{lemma}
% For all $t \geq 1$ and
Algorithm \ref{alg:genupd} satisfies
$\displaystyle{
\frac{1-e^{-\eta}}{\eta}  \bp_{t}^\top \bloss_t -\bq_t^\top \bloss_{t} \leq \frac{1}{\eta} \sum_{i=1}^d q_{i,t}  \ln \left( \frac{v_{i,t}}{\hat p_{i,t+1}} \right)}$ for all $\bq_t \in \Delta_d$.
\end{lemma}

\subsection{Sparse target sequences}
\label{sec:bowa}

The work~\cite{bowa} introduced forecasters that are able to efficiently
compete with the best sequence of experts among all those
sequences that only switch a bounded number of times and also
take a small number of different values. Such ``sparse'' sequences
of experts appear naturally in many applications. In this section
we show that their algorithms in fact work very well in comparison
with a much larger class of sequences $\bu_1,\ldots,\bu_T$
that are ``regular''---that is, $m(\bu_1^T)$, defined
in (\ref{eq:regularity}) is small---and ``sparse'' in the sense
that the quantity $n(\bu_1^T) = \sum_{i=1}^d \max_{t=1,\dots,T} u_{i,t}$
is small. Note that when $\bq_t \in \Delta_d$ for all $t$, then two
interesting upper bounds can be provided. First, denoting the union of
the supports of these convex combinations by $S \subseteq \{ 1,\ldots,d \}$,
we have $n(\b q_1^T) \le |S|$, the cardinality of $S$. Also, $n(\bq_1^T) \leq \bigl| \{ \bq_t, \ \ t = 1,\ldots,T \} \bigr|$,
the cardinality of the pool of convex combinations.
Thus, $n(\bu_1^T)$ generalizes the notion of sparsity
of \cite{bowa}.

Here we consider a family of shared updates of the form
\begin{equation}
\label{eq:shareupdatebw}
\wh{p}_{j,t} = (1-\alpha) v_{j,t} + \alpha \frac{w_{j,t}}{Z_t}\,,
\qquad
    0 \le \alpha \le 1\,,
\end{equation}
where the $w_{j,t}$ are nonnegative weights that may depend on
past and current pre-weights and
$Z_t = \sum_{i=1}^d w_{i,t}$ is a normalization constant.
Shared updates of this form were proposed by
\cite[Sections~3 and~5.2]{bowa}.
Apart from generalizing the regret bounds of \cite{bowa}, we believe
that the analysis given below is significantly simpler and more transparent.
We are also able to slightly improve their original bounds.

We focus on choices of the weights $w_{j,t}$ that satisfy the
following conditions:
there exists a constant $C \geq 1$ such that for all
$j=1,\dots,d$ and $t=1,\dots,T$,
\begin{equation}
\label{eq:cdw}
v_{j,t} \leq w_{j,t} \leq 1 \qquad \text{and} \qquad C \, w_{j,t+1} \geq w_{j,t}~.
\end{equation}
The next result improves on Theorem~\ref{prop:dtv}
when  $T \ll d$ and $n (\b u_1^T) \ll m(\b u_1^T)$, that is,
when the dimension (or number of experts) $d$ is large but
the sequence $\bu_1^T$ is sparse. Its proof can be found in
the supplementary material; it is a variation on the
proof of Theorem~\ref{prop:dtv}.
\begin{theorem}
\label{lem:tubeboundbw1}
Suppose Algorithm~\ref{alg:genupd} is run with the
shared update~(\ref{eq:shareupdatebw}) with weights satisfying
the conditions~\eqref{eq:cdw}.
Then for all $T \geq 1$, for all sequences $\bloss_1,\dots,\bloss_T$
of loss vectors $\bloss_t\in [0,1]^d$, and for all sequences
$\bu_1,\dots,\bu_T\in\R_+^d$,
\begin{align*}
 \sum_{t=1}^T \norm[\bu_t]_1\,\bp_t^{\top}\bloss_t - \sum_{t=1}^T\bu_t^{\top}\bloss_t
\le & \, \frac{n(\b u_1^T) \ln d}{\eta} + \frac{n(\b u_1^T) \, T \ln C}{\eta} + \frac{\eta}{8} \sum_{t=1}^T
\norm[\b u_t]_1 \\
&+ \frac{m(\bu_1^T)}{\eta} \ln \frac{\max_{t \leq T} Z_t}{\alpha} + \frac{\sum_{t=2}^T \norm[\b u_t]_1-m(\bu_1^T)}{\eta} \ln \frac{1}{1-\alpha}~.
\end{align*}
\end{theorem}
Corollaries~8 and~9 of~\cite{bowa} can now be generalized (and even improved); we do so---in the supplementary
material---by showing two
specific instances of the generic update~\eqref{eq:shareupdatebw} that
satisfy~\eqref{eq:cdw}.

\subsection{Online tuning of the parameters}
\label{sec:online}

The forecasters studied above need their parameters $\eta$ and $\alpha$
to be tuned according to various quantities, including the time horizon $T$.
We show here how the trick of~\cite{AuCeGe02} of having these parameters vary over time
can be extended to our setting. For the sake of concreteness
we focus on the fixed-share update, i.e., Algorithm~\ref{alg:genupd} run with the
update~(\ref{eq:shareupdateclassicalfs}).
We respectively replace steps 3 and 4 of its description by
the loss and shared updates
\begin{equation}
\label{eq:timevaryingupdate}
v_{j,t+1} = \wh{p}_{j,t}^{\,\frac{\eta_t}{\eta_{t-1}}} e^{-\eta_t \ell_{j,t}} \!\left/
\sum_{i=1}^d \wh{p}_{i,t}^{\,\frac{\eta_t}{\eta_{t-1}}} e^{-\eta_t \ell_{i,t}} \right.
\qquad \mbox{and} \qquad
p_{j,t+1} = \frac{\alpha_t}{d} + (1-\alpha_t) \, v_{j,t+1} \,,
\end{equation}
for all $t \geq 1$ and all $j \in \{ 1,\dots,d\}$,
where $(\eta_\tau)$ and $(\alpha_\tau)$ are two sequences of positive numbers,
indexed by $\tau \geq 1$. We also conventionally define $\eta_0 = \eta_1$.
Theorem~\ref{prop:dtv} is then adapted in the following way
(when $\eta_t \equiv \eta$ and $\alpha_t \equiv \alpha$,
Theorem~\ref{prop:dtv} is exactly recovered).
\begin{theorem}
\label{prop:dtv-adapt}
The forecaster based on the updates~\eqref{eq:timevaryingupdate} is such that whenever
$\eta_t \leq \eta_{t-1}$ and $\alpha_t \leq \alpha_{t-1}$ for all $t \geq 1$, the following performance bound
is achieved. For all $T \geq 1$, for all sequences $\bloss_1,\dots,\bloss_T$ of loss
vectors $\bloss_t\in [0,1]^d$, and for all $\bu_1,\dots,\bu_T \in \R_+^d$,
\begin{align*}
\sum_{t=1}^T  \norm[\bu_t]_1\,\bp_t^{\top}\bloss_t - \sum_{t=1}^T \bu_t^{\top}\bloss_t
& \le \left( \frac{\norm[\b u_t]_1}{\eta_1} + \sum_{t=2}^T \norm[\b u_t]_1 \left(\frac{1}{\eta_t}- \frac{1}{\eta_{t-1}} \right) \right) \! \ln d \\
& \ \ + \frac{m(\bu_1^T)}{\eta_T}\ln \frac{d(1-\alpha_T)}{\alpha_T} +
\sum_{t=2}^T \frac{\norm[\b u_t]_1}{\eta_{t-1}} \ln \frac{1}{1-\alpha_t}
+ \sum_{t=1}^T \frac{\eta_{t-1}}{8} \, \norm[\b u_t]_1\,.
\end{align*}
\end{theorem}
Due to space constraints, we provide an illustration of this bound only
in the supplementary material.

\subsubsection*{Acknowledgments}

The authors acknowledge support from the French National Research Agency (ANR) under grant
EXPLO/RA (``Exploration--exploitation for efficient resource allocation'') and by the PASCAL2 Network of
Excellence under EC grant no. 506778.

\newpage
\bibliographystyle{unsrt}
\bibliography{CBGaLuSt-bib}

\begin{thebibliography}{10}

\bibitem{hewa01}
M.~Herbster and M.~Warmuth.
\newblock Tracking the best linear predictor.
\newblock {\em Journal of Machine Learning Research}, 1:281--309, 2001.

\bibitem{Zink03}
M.~Zinkevich.
\newblock Online convex programming and generalized infinitesimal gradient
  ascent.
\newblock In {\em Proceedings of the 20th International Conference on Machine
  Learning, ICML 2003}, 2003.

\bibitem{CeLu}
N.~Cesa-Bianchi and G.~Lugosi.
\newblock {\em Prediction, learning, and games}.
\newblock Cambridge University Press, 2006.

\bibitem{hewa98}
M.~Herbster and M.~Warmuth.
\newblock Tracking the best expert.
\newblock {\em Machine Learning}, 32:151--178, 1998.

\bibitem{Vov99}
V.~Vovk.
\newblock Derandomizing stochastic prediction strategies.
\newblock {\em Machine Learning}, 35(3):247--282, Jun. 1999.

\bibitem{bowa}
O.~Bousquet and M.K. Warmuth.
\newblock Tracking a small set of experts by mixing past posteriors.
\newblock {\em Journal of Machine Learning Research}, 3:363--396, 2002.

\bibitem{GyLiLu05a}
A.~Gy\"orgy, T.~Linder, and G.~Lugosi.
\newblock Tracking the best of many experts.
\newblock In {\em Proceedings of the 18th Annual Conference on Learning Theory
  (COLT)}, pages 204--216, Bertinoro, Italy, Jun. 2005. Springer.

\bibitem{hase09}
E.~Hazan and C.~Seshadhri.
\newblock Efficient learning algorithms for changing environments.
\newblock {\em Proceedings of the 26th International Conference of Machine
  Learning (ICML)}, 2009.

\bibitem{disc}
A.~Chernov and F.~Zhdanov.
\newblock Prediction with expert advice under discounted loss.
\newblock In {\em Proceedings of the 21st International Conference on
  Algorithmic Learning Theory, ALT 2010}, pages 255--269. Springer, 2008.

\bibitem{blma07}
A.~Blum and Y.~Mansour.
\newblock From extermal to internal regret.
\newblock {\em Journal of Machine Learning Research}, 8:1307--1324, 2007.

\bibitem{AuCeGe02}
P.~Auer, N.~Cesa-Bianchi, and C.~Gentile.
\newblock Adaptive and self-confident on-line learning algorithms.
\newblock {\em Journal of Computer and System Sciences}, 64:48--75, 2002.

\end{thebibliography}

\newpage
\appendix

% Supplementary material

\section{Online convex optimization on the simplex}
\label{sec:convex}

By using a standard reduction, the results of the main body of the paper (for linear optimization on the simplex)
can be applied to online convex optimization on the simplex. In this setting, at each step $t$ the forecaster chooses $\bp_t\in\Delta_d$ and then is given access to a convex loss $\loss_t : \Delta_d\to [0,1]$. Now, using Algorithm~\ref{alg:genupd} with the loss vector $\bloss_t
\in \partial\loss_t(\bp_t)$ given by a subgradient of $\ell_t$
leads to the desired bounds. Indeed, by the convexity of $\loss_t$, the regret at each time $t$ with respect to any
vector $\bu_t \in \R_+^d$ with $\norm[\bu_t]_1 > 0$ is then bounded as
\[
\norm[\bu_t]_1 \left( \loss_t(\bp_t) -  \loss_t \!\biggl(\frac{\bu_t}{\norm[\bu_t]_1}\biggr) \right)
\le \bigl(\norm[\bu_t]_1\bp_t - \bu_t\bigr)^{\top}\bloss_t\,.
\]

\section{Proof of Theorem~\ref{lem:tubeboundbw1}; application of the bound to
two different updates}

\begin{proof}
The beginning and the end of the proof are similar to the one of Theorem~\ref{prop:dtv},
as they do not depend on the specific weight update. In particular, inequalities \eqref{eq:proofdtv0} and \eqref{eq:proofdtv1} remain the same.
The proof is modified after~\eqref{eq:proofdtv2}, which this time we upper bound using the first condition in~\eqref{eq:cdw},
\begin{multline}
\label{eq:bwdtv2}
\sum_{i=1}^d \left( u_{i,t} \ln \frac{1}{\wh{p}_{i,t}} - u_{i,t-1} \ln \frac{1}{v_{i,t}} \right) =  \sum_{i \,:\, u_{i,t} \geq u_{i,t-1}} \left(u_{i,t}-u_{i,t-1}\right) \ln \frac{1}{\wh{p}_{i,t}} + u_{i,t-1} \ln \frac{v_{i,t}}{\wh{p}_{i,t}} \\
+ \sum_{i \,:\, u_{i,t} < u_{i,t-1}} \underbrace{(u_{i,t}-u_{i,t-1})}_{\leq 0}
\underbrace{\ln \frac{1}{v_{i,t}}}_{\geq \ln (1/w_{i,t})} + u_{i,t} \ln \frac{v_{i,t}}{\wh{p}_{i,t}}~.
\end{multline}
By definition of the shared update~\eqref{eq:shareupdatebw}, we have $1/\wh{p}_{i,t} \leq Z_t/(\alpha \, w_{i,t})$ and
$v_{i,t} / \wh{p}_{i,t} \leq 1/(1-\alpha)$. We then upper bound the quantity at hand in~\eqref{eq:bwdtv2} by
\begin{eqnarray*}
& & \ \sum_{i \,:\, u_{i,t} \geq u_{i,t-1}} \left(u_{i,t}-u_{i,t-1}\right) \ln \left( \frac{Z_t}{\alpha \, w_{i,t}} \right)
+ \left( \sum_{i \,:\, u_{i,t} \geq u_{i,t-1}} u_{i,t-1} + \sum_{i \,:\, u_{i,t} < u_{i,t-1}}u_{i,t} \right) \ln \frac{1}{1-\alpha} \\
& & + \sum_{i \,:\, u_{i,t} < u_{i,t-1}} \left(u_{i,t}-u_{i,t-1}\right) \ln \frac{1}{w_{i,t}} \\
& = & D_{\mathrm{TV}}(\bu_{t},\bu_{t-1}) \ln \frac{Z_t}{\alpha} +  \bigl(\norm[\b u_t]_1 -  D_{\mathrm{TV}}(\bu_{t},\bu_{t-1}) \bigr) \ln \frac{1}{1-\alpha} + \sum_{i=1}^d \left(u_{i,t}-u_{i,t-1}\right) \ln \frac{1}{w_{i,t}}\,.
\end{eqnarray*}
Proceeding as in the end of the proof of Theorem~\ref{prop:dtv}, we then get the claimed bound, provided that we can show that
\[
\sum_{t=2}^T \sum_{i=1}^d \left(u_{i,t}-u_{i,t-1}\right) \ln \frac{1}{w_{i,t}}
\leq n(\b u_1^T) \,(\ln d + T \ln C)  - \norm[\b u_1]_1 \ln d\,,
\]
which we do next. Indeed, the left-hand side can be rewritten as
\begin{eqnarray*}
 & & \sum_{t=2}^T \sum_{i=1}^d \left(u_{i,t} \ln \frac{1}{w_{i,t}} - u_{i,t} \ln \frac{1}{w_{i,t+1}} \right) +
 \sum_{t=2}^T \sum_{i=1}^d \left(u_{i,t} \ln \frac{1}{w_{i,t+1}} - u_{i,t-1} \ln \frac{1}{w_{i,t}} \right) \\
 & \leq & \left( \sum_{t=2}^{T} \sum_{i=1}^d u_{i,t} \ln \frac{C \, w_{i,t+1}}{w_{i,t}} \right) + \left( \sum_{i=1}^d u_{i,T} \ln \frac{1}{w_{i,T+1}}
 - \sum_{i=1}^d u_{i,1} \ln \frac{1}{w_{i,2}} \right)  \\
 & \leq & \left( \sum_{i=1}^d \left( \max_{t=1,\dots,T} u_{i,t} \right)
 \sum_{t=2}^{T} \ln \frac{C \, w_{i,t+1}}{w_{i,t}} \right) +
 \left( \sum_{i=1}^d \left( \max_{t=1,\ldots,T} u_{i,t} \right) \ln \frac{1}{w_{i,T+1}}
 - \sum_{i=1}^d u_{i,1} \ln \frac{1}{w_{i,2}} \right) \\
 & = & \sum_{i=1}^d \left( \max_{t=1,\dots,T} u_{i,t} \right) \left( (T-1)\ln C + \ln \frac{1}{w_{i,2}} \right)
 - \sum_{i=1}^d u_{i,1} \ln \frac{1}{w_{i,2}}\,,
\end{eqnarray*}
where we used $C \geq 1$ for the first inequality and the second condition in~\eqref{eq:cdw} for the second inequality.
The proof is concluded by noting that \eqref{eq:cdw} entails
$w_{i,2} \geq (1/C) w_{i,1} \geq (1/C) v_{i,1} = 1/(dC)$ and that the coefficient
$\max_{t=1,\dots,T} u_{i,t} - u_{i,1}$ in front of $\ln (1/w_{i,2})$ is nonnegative.
\end{proof}

The first update uses
$w_{j,t} = \max_{s \leq t} v_{j,s}$. Then~\eqref{eq:cdw} is satisfied with
$C =1$. Moreover, since a sum of maxima of nonnegative elements is smaller than
the sum of the sums, $Z_t \leq \min \{d,t\} \leq T$.
This immediately gives the following result.
\begin{corollary}
\label{co:bw}
Suppose Algorithm~\ref{alg:genupd} is run with the
update~\eqref{eq:shareupdatebw} with $w_{j,t} = \max_{s \leq t} v_{j,s}$.
For all $T \geq 1$, for all sequences $\bloss_1,\dots,\bloss_T$
of loss vectors $\bloss_t\in [0,1]^d$, and for all
$\bq_1,\dots,\bq_T \in \Delta_d$,
\begin{align*}
\sum_{t=1}^T \bp_t^{\top}\bloss_t - \sum_{t=1}^T \bq_t^{\top}\bloss_t
\le & \, \frac{n(\bq_1^T) \ln d}{\eta}
+ \frac{\eta}{8} T + \frac{m(\bq_1^T)}{\eta} \ln \frac{T}{\alpha}
+ \frac{T-m(\bq_1^T)-1}{\eta} \ln \frac{1}{1-\alpha}\,.
\end{align*}
\end{corollary}

The second update we discuss
uses $w_{j,t} = \max_{s \leq t} e^{\gamma(s-t)} v_{j,s}$
in~\eqref{eq:shareupdatebw} for some $\gamma > 0$.
Both conditions in \eqref{eq:cdw} are satisfied with $C = e^\gamma$. One also has that
\[
Z_t \leq d \qquad \mbox{and} \qquad Z_t \leq \sum_{\tau \geq 0} e^{-\gamma \tau} = \frac{1}{1 - e^{-\gamma}} \leq \frac{1}{\gamma}
\]
as $e^x \geq 1+x$ for all real $x$. The bound of Theorem~\ref{lem:tubeboundbw1} then instantiates
as
\[
\frac{n(\bq_1^T) \ln d}{\eta} + \frac{n(\bq_1^T) \, T \gamma}{\eta} + \frac{\eta}{8} T
+ \frac{m(\bq_1^T)}{\eta} \ln \frac{\min \{ d, \,1/\gamma \}}{\alpha} +
\frac{T - m(\bq_1^T) - 1}{\eta} \ln \frac{1}{1-\alpha}
\]
when sequences $\bu_t = \bq_t \in \Delta_d$ are considered. This bound is best understood
when $\gamma$ is tuned optimally based on $T$ and on two bounds $m_0$ and $n_0$
over the quantities $m(\bq_1^T)$ and $n(\bq_1^T)$. Indeed,
by optimizing $n_0 T \gamma + m_0 \ln (1/\gamma)$, i.e.,
by choosing $\gamma = m_0 / (n_0 \, T)$, one gets a bound
that improves on the one of the previous corollary:

\begin{corollary}
\label{co:bw-pastdec}
Let $m_0,\,n_0 > 0$.
Suppose Algorithm~\ref{alg:genupd} is run with the
update~$w_{j,t} = \max_{s \leq t} e^{\gamma(s-t)} v_{j,s}$ where
$\gamma = m_0 / (n_0 \, T)$.
For all $T \geq 1$, for all sequences $\bloss_1,\dots,\bloss_T$
of loss vectors $\bloss_t\in [0,1]^d$, and for all
$\bq_1,\dots,\bq_T \in \Delta_d$ such that $m(\bq_1^T) \leq m_0$
and $n(\bq_1^T) \le n_0$, we have
\begin{eqnarray*}
\sum_{t=1}^T \bp_t^{\top}\bloss_t - \sum_{t=1}^T \bq_t^{\top}\bloss_t
& \leq & \frac{n_0 \ln d}{\eta}
+ \frac{m_0}{\eta} \left( 1 + \ln \, \min \left\{ d, \, \frac{n_0 \, T}{m_0} \right\} \right) \\
& & \ \ \ + \frac{\eta}{8} T
+ \frac{m_0}{\eta} \ln \frac{1}{\alpha} +
\frac{T - m_0 - 1}{\eta} \ln \frac{1}{1-\alpha}\,.
\end{eqnarray*}
\end{corollary}

As the factors $e^{-\gamma t}$ cancel out in the numerator and
denominator of the ratio in~\eqref{eq:shareupdatebw},
there is a straightforward implementation of the algorithm (not requiring the
knowledge of $T$) that needs to maintain only $d$ weights.

In contrast, the corresponding algorithm of \cite{bowa}, using the updates
$\wh{p}_{j,t} = (1-\alpha) v_{j,t} + \alpha S_t^{-1}\sum_{s \leq t-1}
(s-t)^{-1} v_{j,s}$ or
$\wh{p}_{j,t} = (1-\alpha) v_{j,t} + \alpha S_t^{-1} \max_{s \leq t-1}
(s-t)^{-1} v_{j,s}$,
where $S_t$ denote normalization factors, needs
to maintain $O(dT)$ weights with a naive implementation, and
$O(d \ln T)$ weights with a more sophisticated one. In addition,
the obtained bounds are slightly worse than the one stated
above in Corollary~\ref{co:bw-pastdec} as an
additional factor of $m_0 \ln(1+\ln T)$ is present in
\cite[Corollary~9]{bowa}.

\section{Proof of Theorem~\ref{prop:dtv-adapt}; illustration of the obtained bound}

We first adapt Lemma~\ref{lem:b0}.

\begin{lemma}
\label{lem:b0adapt}
The forecaster based on the loss and shared updates \eqref{eq:timevaryingupdate}
satisfies, for all $t\geq 1$ and for all $\bq_t\in\Delta_d$,
\[
\bigl(\bp_t - \bq_t\bigr)^{\! \top}\bloss_t
\le \sum_{i=1}^d q_{i,t} \left( \frac{1}{\eta_{t-1}} \ln \frac{1}{\wh{p}_{i,t}} -
\frac{1}{\eta_t} \ln \frac{1}{v_{i,t+1}} \right) +
\left(\frac{1}{\eta_t}- \frac{1}{\eta_{t-1}}\right) \! \ln d + \frac{\eta_{t-1}}{8}\,,
\]
whenever $\eta_t \leq \eta_{t-1}$.
\end{lemma}

\begin{proof}
By Hoeffding's inequality,
\[
% \label{eq:convexhoeffding-bis}
    \sum_{j=1}^d \wh{p}_{j,t}\,\loss_{j,t}
    \le
    - \frac{1}{\eta_{t-1}} \ln \left( {\sum_{j=1}^d \wh{p}_{j,t}\, e^{-\eta_{t-1}\,\loss_{j,t}}} \right) + \frac{\eta_{t-1}}{8}\,.
\]
By Jensen's inequality, since $\eta_t \leq \eta_{t-1}$ and thus $x \mapsto x^\frac{\eta_{t-1}}{\eta_{t}}$ is convex,
\[
\frac{1}{d} \sum_{j=1}^d \wh{p}_{j,t} \, e^{-\eta_{t-1} \ell_{j,t}} = \frac{1}{d} \sum_{j=1}^d  \left( \wh{p}_{j,t}^{\,\frac{\eta_t}{\eta_{t-1}}}
e^{-\eta_t \ell_{j,t}} \right)^{\! \frac{\eta_{t-1}}{\eta_t}} \geq \left(\frac{1}{d} \sum_{j=1}^d \wh{p}_{j,t}^{\,\frac{\eta_t}{\eta_{t-1}}}
e^{-\eta_t \ell_{j,t}} \right)^{\! \frac{\eta_{t-1}}{\eta_t}} \,.
\]
Substituting in Hoeffding's bound we get
\[
\bp_t^{\top} \bloss_t \leq
- \frac{1}{\eta_t} \ln \left( \sum_{j=1}^d \wh{p}_{j,t}^{\,\frac{\eta_t}{\eta_{t-1}}} e^{-\eta_t \ell_{j,t}} \right) + \left(\frac{1}{\eta_t}- \frac{1}{\eta_{t-1}}\right) \! \ln d + \frac{\eta_{t-1}}{8}\,.
\]
Now, by definition of the loss update in~\eqref{eq:timevaryingupdate}, for all $i \in \{ 1,\dots,d\}$,
\[
\sum_{j=1}^d \wh{p}_{j,t}^{\,\frac{\eta_t}{\eta_{t-1}}} e^{-\eta_t \ell_{j,t}}
= \frac{1}{v_{i,t+1}} \, \wh{p}_{i,t}^{\,\frac{\eta_t}{\eta_{t-1}}} e^{-\eta_t \ell_{i,t}}\,,
\]
which, after substitution in the previous bound leads to the inequality
\[
\bp_t^{\top} \bloss_t \leq
\ell_{i,t} + \frac{1}{\eta_{t-1}} \ln \frac{1}{\wh{p}_{i,t}} -
\frac{1}{\eta_t} \ln \frac{1}{v_{i,t+1}} + \left(\frac{1}{\eta_t}- \frac{1}{\eta_{t-1}}\right) \! \ln d + \frac{\eta_{t-1}}{8}\,,
\]
valid for all $i \in \{ 1,\dots,d\}$.
The proof is concluded by taking a convex aggregation over $i$ with respect to $\bq_t$.
\end{proof}

The proof of Theorem~\ref{prop:dtv-adapt} follows the steps of the one
of Theorem~\ref{prop:dtv}; we sketch it below. \\

\noindent
\emph{Proof of Theorem~\ref{prop:dtv-adapt}.}
Applying Lemma~\ref{lem:b0adapt}
with $\b q_t = \b u_t / \norm[\b u_t]_1$, and multiplying by $\norm[\b u_t]_1$, we get for all $t \geq 1$ and $\bu_t \in \R_+^d$,
\begin{multline}
\label{eq:deb1}
\norm[\b u_t]_1 \,\bp_t^{\top}\bloss_t - \bu_t^{\top}\bloss_t
\le
\frac{1}{\eta_{t-1}} \sum_{i=1}^d u_{i,t} \ln \frac{1}{\wh{p}_{i,t}} -
\frac{1}{\eta_t} \sum_{i=1}^d u_{i,t} \ln \frac{1}{v_{i,t+1}} \\
+ \norm[\b u_t]_1 \left(\frac{1}{\eta_t}- \frac{1}{\eta_{t-1}}\right) \! \ln d + \frac{\eta_{t-1}}{8} \, \norm[\b u_t]_1\,.
\end{multline}
We will sum these bounds over $t \geq 1$ to get the desired result but need to perform first some additional boundings
for $t \geq 2$; in particular, we examine
\begin{multline}
\label{eq:deb2}
\frac{1}{\eta_{t-1}} \sum_{i=1}^d u_{i,t} \ln \frac{1}{\wh{p}_{i,t}} -
\frac{1}{\eta_t} \sum_{i=1}^d u_{i,t} \ln \frac{1}{v_{i,t+1}} \\
=
\frac{1}{\eta_{t-1}} \sum_{i=1}^d \left( u_{i,t} \ln \frac{1}{\wh{p}_{i,t}} - u_{i,t-1} \ln \frac{1}{v_{i,t}} \right)
+ \sum_{i=1}^d \left( \frac{u_{i,t-1}}{\eta_{t-1}} \ln \frac{1}{v_{i,t}} - \frac{u_{i,t}}{\eta_t} \ln \frac{1}{v_{i,t+1}} \right)\,,
\end{multline}
where the first difference in the right-hand side can be bounded as in~\eqref{eq:proofdtv2} by
\begin{eqnarray}
\nonumber
\lefteqn{\sum_{i=1}^d \left( u_{i,t} \ln \frac{1}{\wh{p}_{i,t}} - u_{i,t-1} \ln \frac{1}{v_{i,t}} \right)} \\
\nonumber
& \leq & \sum_{i\,:\,u_{i,t} \geq u_{i,t-1}} \left( \left(u_{i,t}-u_{i,t-1}\right) \ln \frac{1}{\wh{p}_{i,t}} + u_{i,t-1} \ln \frac{v_{i,t}}{\wh{p}_{i,t}} \right) + \sum_{i\,:\,u_{i,t} < u_{i,t-1}} u_{i,t} \ln \frac{v_{i,t}}{\wh{p}_{i,t}} \\
\nonumber
& \leq & D_{TV}(\bu_{t},\bu_{t-1}) \ln \frac{d}{\alpha_t} + \bigl(\norm[\b u_t]_1
- {D_{TV}(\bu_{t},\bu_{t-1})} \bigr) \ln \frac{1}{1-\alpha_t} \\
\label{eq:deb3}
& \leq & D_{TV}(\bu_{t},\bu_{t-1}) \ln \frac{d(1-\alpha_T)}{\alpha_T} + \norm[\b u_t]_1 \ln \frac{1}{1-\alpha_t}\,,
\end{eqnarray}
where we used for the second inequality that the shared update in~\eqref{eq:timevaryingupdate} is such that
$1/\wh{p}_{i,t} \leq d/\alpha_t$ and $v_{i,t}/\wh{p}_{i,t} \leq 1/(1-\alpha_t)$,
and for the third inequality, that $\alpha_t \geq \alpha_T$ and $x \mapsto (1-x)/x$ is increasing
on $(0,1]$. Summing~\eqref{eq:deb2} over $t = 2,\ldots,T$ using~\eqref{eq:deb3}
and the fact that $\eta_t \geq \eta_T$, we get
\begin{eqnarray*}
\lefteqn{\sum_{t=2}^T \left( \frac{1}{\eta_{t-1}} \sum_{i=1}^d u_{i,t} \ln \frac{1}{\wh{p}_{i,t}} -
\frac{1}{\eta_t} \sum_{i=1}^d u_{i,t} \ln \frac{1}{v_{i,t+1}} \right)} \\
& \leq & \frac{m(\bu_1^T)}{\eta_T}\ln \frac{d(1-\alpha_T)}{\alpha_T} +
\sum_{t=2}^T \frac{\norm[\b u_t]_1}{\eta_{t-1}} \ln \frac{1}{1-\alpha_t}
+ \sum_{i=1}^d \biggl( \frac{u_{i,1}}{\eta_{1}} \ln \frac{1}{v_{i,2}} - \underbrace{\frac{u_{i,T}}{\eta_T}
\ln \frac{1}{v_{i,T+1}}}_{\geq 0} \biggr).
\end{eqnarray*}
An application of~\eqref{eq:deb1}~---including for $t=1$, for which we recall that $\wh{p}_{i,1} = 1/d$ and
$\eta_1 = \eta_0$ by convention--- concludes the proof. \hfill $\square$

We now instantiate the obtained bound to the case
of, e.g., $T$--adaptive regret guarantees, when $T$ is unknown and/or can increase without bounds.
\begin{corollary}
\label{cor:adapt-adapt}
The forecaster based on the updates discussed above with
$\eta_t = \sqrt{\bigl(\ln(dt)\bigr)/t}$ for $t \geq 3$ and
$\eta_0 = \eta_1 = \eta_2 = \eta_3$ on the one hand, $\alpha_t = 1/t$ on the other hand, is such
that for all $T \geq 3$ and for all sequences $\bloss_1,\dots,\bloss_T$ of loss
vectors $\bloss_t\in [0,1]^d$,
\[
\max_{\scriptsize [r,s] \subset [1,T]} \left\{ \sum_{t=r}^{s} \bp_t^{\top}\bloss_t
 - \min_{\bq \in \Delta_d} \sum_{t=r}^{s} \bq^{\top}\bloss_t \right\}
 \leq \sqrt{2 T \ln (dT)} + \sqrt{3 \ln (3d)}\,.
\]
\end{corollary}
\begin{proof}
The sequence $n \mapsto \ln(n)/n$ is only non-increasing after round $n \geq 3$,
so that the defined sequences of $(\alpha_t)$ and $(\eta_t)$ are non-increasing, as desired.
For a given pair $(r,s)$ and a given $\bq \in \Delta_d$, we consider the
sequence $\nu_1^T$ defined in the proof of Corollary~\ref{cor:adapt};
it satisfies that $m(\bu_1^T) \leq 1$ and $\norm[\bu_t]_1 \leq 1$ for all $t \geq 1$.
Therefore, Theorem~\ref{prop:dtv-adapt} ensures that
\[
\sum_{t=r}^{s} \bp_t^{\top}\bloss_t  - \min_{\bq \in \Delta_d} \sum_{t=r}^{s} \bq^{\top}\bloss_t
\leq \frac{\ln d}{\eta_T} + \frac{1}{\eta_T} \ln \underbrace{\frac{d(1-\alpha_T)}{\alpha_T}}_{\leq dT}
+ \!\!\!\! \underbrace{\sum_{t=2}^T \frac{1}{\eta_{t-1}} \ln \frac{1}{1-\alpha_t}}_{\leq (1/\eta_T) \sum_{t=2}^T \ln (t/(t-1))
= (\ln T)/\eta_T} \!\!\!\! + \sum_{t=1}^T \frac{\eta_{t-1}}{8}\,.
\]
It only remains to substitute the proposed values of $\eta_t$ and to note that
\[
\sum_{t=1}^T \eta_{t-1} \leq 3\eta_3 + \sum_{t=3}^{T-1} \frac{1}{\sqrt{t}} \, \sqrt{\ln (dT)}
\leq 3 \sqrt{\frac{\ln(3d)}{3}} + 2 \sqrt{T} \, \sqrt{\ln (dT)}\,.
\]
\vspace{-1.2cm}

\end{proof}

\section{Proof of Theorem \ref{th:HW}}

We recall that the forecaster at hand is the one described in Algorithm~\ref{alg:1},
with the shared update $\bp_{t+1} = \psi_{t+1} \bigl(\uV \bigr)$ for
\begin{equation}
\label{eq:shHW}
\psi_{t+1} \bigl(\uV \bigr) \in \argmin_{\bx \in \Delta_d^\alpha} \mathcal{K}(\bx,\bv_{t+1})\,,
\qquad \mbox{where} \quad \mathcal{K}(\bx,\bv_{t+1}) = \sum_{i=1}^d x_i \ln \frac{x_i}{v_{i,t+1}}
\end{equation}
is the Kullback-Leibler divergence and $\Delta_d^\alpha = [{\alpha}/{d},1]^d \cap \Delta_d$ is the simplex
of convex vectors with the constraint that each component be larger than $\alpha/d$.

The proof of the performance bound starts with an extension of Lemma~\ref{lem:b0}.

\begin{lemma}
\label{lem:hw}
For all $t \geq 1$ and for all $\bq_t \in \Delta_d^\alpha $, the generalized forecaster with
the shared update~\eqref{eq:shHW} satisfies
$$ (\bp_t - \bq_t)^\top \bloss_t \leq \frac{1}{\eta} \sum_{i=1}^d q_{i,t} \ln \frac{\hat p_{i,t+1}}{\wh p_{i,t}} + \frac{\eta}{8} \,.$$
\end{lemma}

\begin{proof}
We rewrite the bound of Lemma~\ref{lem:b0} in terms of Kullback-Leibler divergences,
\begin{eqnarray*}
(\bp_t - \bq_t)^\top \bloss_t  &\leq & \frac{1}{\eta} \sum_{i=1}^d q_{i,t} \ln \frac{v_{i,t+1}}{p_{i,t}} + \frac{\eta}{8}
	 =  \ \frac{\mathcal{K}(\bq_t,\bp_t) - \mathcal{K}(\bq_t,\bv_{t+1})}{\eta} + \frac{\eta}{8} \\
	& \leq & \frac{\mathcal{K}(\bq_t,\bp_t) - \mathcal{K}(\bq_t,\bp_{t+1})}{\eta} + \frac{\eta}{8}
= \frac{1}{\eta} \sum_{i=1}^d q_{i,t} \ln \frac{\hat p_{i,t+1}}{\wh p_{i,t}} + \frac{\eta}{8}\,,
\end{eqnarray*}
where the last inequality holds by applying a generalized Pythagorean theorem for Bregman divergences (here,
the Kullback-Leibler divergence) ---see, e.g., \cite[Lemma 11.3]{CeLu}.
\end{proof}

\begin{proof}
Let $\displaystyle{\bq_t = \frac{\alpha}{d} + (1-\alpha) \frac{\bu_t}{\norm[\bu_t]_1} \in \Delta_d^\alpha}$. We have by rearranging the terms for all $t$,
\begin{align*}
\bigl( \norm[\bu_t]_1 \bp_t - \bu_t \bigr)^\top \bloss_t & =
\norm[\bu_t]_1 \left(\bp_t-\bq_t\right)^\top \bloss_t + \left( \frac{\alpha}{d} \norm[\bu_t]_1 - \alpha \bu_t \right)^\top \bloss_t \\
& \leq \norm[\bu_t]_1 \left(\bp_t-\bq_t\right)^\top \bloss_t + {\alpha }\norm[\bu_t]_1 \,.
\end{align*}
Therefore, by applying Lemma~\ref{lem:hw} with $\bq_t \in \Delta_d^\alpha$, we further upper bound the quantity of interest as
\begin{eqnarray*}
\bigl( \norm[\bu_t]_1 \bp_t - \bu_t \bigr)^\top \bloss_t
\leq \frac{\norm[\bu_t]_1}{\eta} \sum_{i=1}^d q_{i,t} \ln \frac{\wh{p}_{i,t+1}}{\wh{p}_{i,t}} + \frac{\eta}{8}\norm[\bu_t]_1  + {\alpha }\norm[\bu_t]_1 \,.
\end{eqnarray*}
The upper bound is rewritten by summing over $t$ and applying an Abel transform to its first term,
\begin{eqnarray*}
& & \sum_{t=1}^T \frac{\norm[\bu_t]_1}{\eta} \sum_{i=1}^d q_{i,t} \ln \frac{\wh{p}_{i,t+1}}{\wh{p}_{i,t}} + \frac{\eta}{8}\norm[\bu_t]_1  + {\alpha }\norm[\bu_t]_1 \\
& = & \frac{\norm[\bu_1]_1 \ln d}{\eta} + \frac{\norm[\bu_T]_1}{\eta} \underbrace{\sum_{i=1}^d q_{i,T} \ln \wh{p}_{i,T+1}}_{\leq 0} +
\frac{1}{\eta} \sum_{t=2}^{T} \sum_{i=1}^d \underbrace{\bigl( \norm[\bu_{t}]_1 q_{i,t} - \norm[\bu_{t-1}] q_{i,t-1} \bigr)}_{=  (1-\alpha) (u_{i,t} - u_{i,t-1})}
\underbrace{\ln \frac{1}{\hat{p}_{i,t}}}_{0 \le \ \cdot \ \leq \ln \frac{d}{\alpha}} \\
& & + \left(\frac{\eta}{8} + {\alpha }\right) \sum_{t=1}^T\norm[\bu_t]_1 \\
& \leq & \frac{\norm[\bu_1]_1 \ln d}{\eta} + \frac{1-\alpha}{\eta} \left( \sum_{t=2}^T
	{D_{\mathrm{TV}}(\bu_t,\bu_{t-1})} \right) \ln \frac{d}{\alpha} + \left(\frac{\eta}{8} + {\alpha }\right) \sum_{t=1}^T\norm[\bu_t]_1 \,.
\end{eqnarray*}

\end{proof}

\end{document}